\documentclass{article}

\usepackage{fullpage}
\usepackage[utf8]{inputenc}
\usepackage[T1]{fontenc}
\usepackage{url}
\usepackage{booktabs}
\usepackage{amsfonts}
\usepackage{amsmath}
\usepackage{amssymb}
\usepackage{amsthm}
\usepackage{nicefrac}
\usepackage{microtype}
\usepackage{xcolor}
\usepackage{natbib}

\usepackage[ruled]{algorithm2e}
\usepackage{algorithmic}

\usepackage{todonotes}

\usepackage[hidelinks]{hyperref}
\usepackage[capitalise,noabbrev]{cleveref}

\crefname{equation}{}{}
\newtheorem{theorem}{Theorem}
\newtheorem{lemma}{Lemma}
\newtheorem{proposition}{Proposition}
\newtheorem{corollary}{Corollary}

\newtheorem{remk}[theorem]{Remark}
\newtheorem{exmp}[theorem]{Example}

\def\FullBox{\hbox{\vrule width 8pt height 8pt depth 0pt}}

\def\qed{\ifmmode\qquad\FullBox\else{\unskip\nobreak\hfil
\penalty50\hskip1em\null\nobreak\hfil\FullBox
\parfillskip=0pt\finalhyphendemerits=0\endgraf}\fi}

\def\qedsketch{\ifmmode\Box\else{\unskip\nobreak\hfil
\penalty50\hskip1em\null\nobreak\hfil$\Box$
\parfillskip=0pt\finalhyphendemerits=0\endgraf}\fi}

\newcommand{\R}{{\mathbb R}} 

\DeclareMathOperator*{\argmin}{arg\,min}

\newcommand{\lr}[1]{\left (#1\right)} 
\newcommand{\lrs}[1]{\left [#1 \right]} 
\newcommand{\lrc}[1]{\left \{#1\right\}} 
\newcommand{\lrn}[1]{\left \|#1\right\|} 

\let\P\undefined
\NewDocumentCommand{\P}{o}{\mathbb P{\IfValueT{#1}{\lr{#1}}}}

\NewDocumentCommand{\E}{o}{\mathbb E\IfValueT{#1}{\lrs{#1}}}

\DeclareMathOperator{\V}{Var}
\NewDocumentCommand{\Var}{o}{\V\IfValueT{#1}{\lrs{#1}}}

\NewDocumentCommand{\1}{o}{\mathds 1{\IfValueT{#1}{\lr{#1}}}}

\NewDocumentCommand{\CO}{o}{\mathcal O{\IfValueT{#1}{\lr{#1}}}}

\newcommand{\dd}{\mathop{}\!\mathrm{d}}

\title{Time-Uniform Self-Normalized Concentration for Discounted Least Squares: Limits and Corrections}
\author{%
  Yi-Shan Wu\thanks{Research Center for Information Technology Innovation, Academia Sinica, Taipei, Taiwan. \texttt{yishan.eason.wu@gmail.com}}%
}
\date{}

\begin{document}
\maketitle
\begin{abstract}
Self-normalized concentration inequalities are standard tools in bandit and reinforcement-learning analyses. A widely used weighted extension claims an analogous time-uniform guarantee for discounted least-squares estimators in non-stationary problems. A simple scalar Gaussian counterexample with a fixed parameter shows that the claimed bounded radius is crossed with probability one. For fixed discount and regularization parameters, we further show that, when \(\delta\leq1/2\) and \(T/\delta\) is sufficiently large, any deterministic anytime boundary valid uniformly over the stated conditionally sub-Gaussian model class must be at least of order \(R\sqrt{\log(T/\delta)}\) at some time by horizon \(T\); for nondecreasing boundaries, this order is required at time \(T\). We identify the proof error: different terminal times use different Gaussian mixing distributions, so the fixed-time mixtures do not form one supermartingale, and the stopping-time argument does not repair this failure. Finally, we show that the weighted inequality remains valid at each fixed deterministic time, give valid finite- and infinite-horizon corrections, and discuss consequences for downstream analyses.
\end{abstract}

\section{Introduction}\label{Sec:Intro}
Self-normalized concentration inequalities are a basic tool in sequential learning. In linear bandits and online regression, covariates may be chosen using past observations, so the accumulated noise does not have one fixed variance scale. Instead, its scale depends on the directions and amount of information collected so far. Self-normalized inequalities control this noise after normalizing it by the corresponding empirical covariance matrix. They are the usual route to confidence sets for an unknown fixed linear parameter \(\theta^*\); the classical linear-bandit confidence construction of \citet{APS11-LinUCB} is a standard example. A time-uniform guarantee is often preferred because subsequent decisions and stopping times may depend on the observed data; it provides one high-probability event that remains valid throughout the trajectory.

In a non-stationary problem, the parameter to be estimated may change over time. Then old observations can become less informative about the current parameter \(\theta_t\), and it is natural to give recent data more weight than older data. Discounted or weighted least-squares estimators implement exactly this idea. They are used in non-stationary linear and generalized-linear bandits, where discounting provides a simple way to track a changing parameter; see, for example, \citet{RVC19weightedLB,RFCG21-GLB,WZZ23revisiting}. To analyze such methods, one needs a confidence bound for the corresponding weighted noise process that holds along the whole trajectory. Under terminal-time discounting, however, the weights assigned to past observations change as time moves forward. Thus, a bound valid at each fixed time does not automatically give one event that is valid simultaneously over all times.

This note examines a widely used claimed time-uniform weighted self-normalized inequality in discounted bandit and reinforcement-learning analyses. We show by a scalar Gaussian example that its bounded radius is false and derive a lower bound for every horizon \(T\). For fixed discount and regularization parameters, this lower bound requires the running envelope of any valid anytime boundary to be at least of order \(R\sqrt{\log(T/\delta)}\) whenever \(\delta\leq1/2\) and \(T/\delta\) is sufficiently large. In particular, for nondecreasing deterministic radius functions, the radius at time \(T\) itself must grow at this rate. We also identify the failure in the proposed stopping-time argument, clarify that the weighted inequality remains valid at each fixed deterministic time, and give valid finite- and infinite-horizon corrections through explicit confidence allocation.

\paragraph{Relation to prior work.} \citet{APS11-LinUCB} established the standard anytime self-normalized inequality for the unweighted process with a fixed ridge regularizer, whereas \citet{RVC19weightedLB} proposed the widely used weighted extension whose discounted specialization motivates this note. Our contribution is a correction-and-consequences analysis: we distinguish fixed-time from time-uniform validity, give a counterexample and a matching-order lower bound for the discounted process, and trace how the claimed anytime form enters later analyses. \Cref{Sec:Impact} provides a more detailed discussion of the relevant literature, including direct downstream uses of the claimed inequality and later treatments based on fixed-time or finite-horizon control.

\paragraph{Organization.} The note is organized as follows. \Cref{Sec:background} introduces the classical self-normalized inequality, the claimed discounted extension, and valid finite- and infinite-horizon corrections. \Cref{Sec:Counterexample} constructs a scalar Gaussian counterexample to the time-uniform inequality claimed by \citet{RVC19weightedLB} and gives lower bounds for valid anytime boundaries. \Cref{Sec:SN-technique-revisit} revisits the proof of the classical self-normalized inequality and presents a valid extension to one predictable sequence of weights. \Cref{Sec:Russac} identifies the invalid step in the proposed time-uniform argument and explains what remains valid at a fixed deterministic time. Finally, \cref{Sec:Impact} discusses the scope of the issue and its consequences for downstream analyses in bandit and reinforcement-learning literature.

\section{Background and the Claimed Inequality}\label{Sec:background}
\paragraph{Stationary ridge estimation.}
We first recall how self-normalized quantities arise in the usual stationary linear model when performing sequential decision-making. At each time $t$, the learner chooses a feature vector $X_t\in\mathbb R^d$ and observes $Y_t=X_t^\top\theta^*+\eta_t$, where $\theta^*\in\mathbb R^d$ is an unknown but fixed parameter and $\eta_t$ is noise. For $\lambda>0$, the ridge estimator is the solution of
\[
    \widehat\theta_t\triangleq\argmin_{\theta\in\mathbb R^d}
    \left\{
        \sum_{s=1}^t\bigl(Y_s-X_s^\top\theta\bigr)^2
        +\lambda\|\theta\|_2^2
    \right\}.
\]
Its equivalent closed-form expression is
\[
    \widehat\theta_t\triangleq V_t^{-1}\sum_{s=1}^t X_sY_s,
    \quad \text{where} \quad
    V_t\triangleq\lambda I_d+\sum_{s=1}^tX_sX_s^\top.
\]
Substituting the model into the closed-form estimator yields the error decomposition
\[
    \widehat\theta_t-\theta^*=V_t^{-1}\bigl(S_t-\lambda\theta^*\bigr),
    \quad \text{where} \quad
    S_t\triangleq\sum_{s=1}^t \eta_sX_s,
\]
Measuring error in the data-dependent norm $\|u\|_{V_t}\triangleq\sqrt{u^\top V_tu}$, which reflects the information collected in each direction, gives
 \[
    \|\widehat\theta_t-\theta^*\|_{V_t}
    =
    \left\|\bigl(S_t-\lambda\theta^*\bigr)\right\|_{V_t^{-1}}
    \leq
    \|S_t\|_{V_t^{-1}}+\|\lambda\theta^*\|_{V_t^{-1}}
    \leq
    \|S_t\|_{V_t^{-1}}+\sqrt{\lambda}\|\theta^*\|_2.
\]
The second term is a constant, and the derivation is due to $V_t\succeq \lambda I_d$, and $S_t$ is the accumulated noise in the estimator, while $V_t$ records the amount and directions of information collected so far. Hence, constructing a confidence set for $\theta^*$ reduces to controlling the self-normalized noise term $\|S_t\|_{V_t^{-1}}$.

To state the sequential assumptions precisely, let $(\mathcal F_t)_{t\geq 0}$ be a filtration. Assume that $X_t$ is $\mathcal F_{t-1}$-measurable and satisfies $\|X_t\|_2\leq L$, while $\eta_t$ is $\mathcal F_t$-measurable and conditionally $R$-sub-Gaussian given $\mathcal F_{t-1}$.
Here $L,R>0$, and fix $\delta\in(0,1)$. Under these assumptions, the self-normalized inequality of \citet{APS11-LinUCB}, with ridge regularization $\lambda I_d$, gives the time-uniform statement
\[
\P\lr{\exists t\geq 0: \lrn{S_t}_{V_t^{-1}}>R\sqrt{2\log\frac{\det(V_t)^{1/2}}{\delta\,\lambda^{d/2}}}}
\leq \delta.
\]
Moreover,
\[
\frac{\det(V_t)^{1/2}}{\lambda^{d/2}}\le \lr{1+\frac{tL^2}{\lambda d}}^{d/2}.
\]
Hence,
\begin{equation}\label{Eq:AY-Plugin}
\P\lr{\exists t\geq 0: \lrn{S_t}_{V_t^{-1}}>R\sqrt{2\log\frac{1}{\delta}+d\log\lr{1+\frac{tL^2}{\lambda d}}}}
\leq \delta.
\end{equation}
The radius in \cref{Eq:AY-Plugin} grows approximately as
$R\sqrt{d\log t}$.

\paragraph{Discounted estimation.}
In a non-stationary problem, the parameter may instead change with time, so that $Y_s=X_s^\top\theta_s+\eta_s$. Typically the parameter is assumed to change slowly, so that $\theta_t$ is close to $\theta_{t-1}$. In this case, it is natural to give recent observations more weight than older ones. A simple way to implement this idea is through exponential discounting. For a fixed discount factor $0<\gamma<1$, the corresponding weighted ridge estimator and the weighted design matrix are
\[
    \widehat\theta_t^\gamma\triangleq (V_t^\gamma)^{-1}
    \sum_{s=1}^t\gamma^{t-s}X_sY_s,
    \qquad
    V_t^\gamma\triangleq \lambda I_d+\sum_{s=1}^t\gamma^{t-s}X_sX_s^\top,
    \qquad 0<\gamma<1.
    \]
The random noise contribution to this estimator and its self-normalizer are
\[
    S_t^\gamma\triangleq\sum_{s=1}^t\gamma^{t-s}\eta_sX_s,
    \qquad
    \widetilde V_t^\gamma\triangleq \lambda I_d+\sum_{s=1}^t\gamma^{2(t-s)}X_sX_s^\top.
\]
The weights in $\widetilde V_t^\gamma$ are squared because it normalizes the variance proxy of the weighted sum $S_t^\gamma$. Thus, it is distinct from the estimator's discounted covariance matrix $V_t^\gamma$, which uses the unsquared weights $\gamma^{t-s}$.
Corollary~3 of \citet{RVC19weightedLB}, derived from their Proposition~1, claims the time-uniform bound
\begin{equation}\label{Eq:Russac}
    \mathbb P\left(\exists t\geq 1:\|S_t^\gamma\|_{(\widetilde V_t^\gamma)^{-1}}> r_t^\gamma(\delta)\right)\leq\delta,
\end{equation}
where for any $\delta\in(0,1)$,
\begin{equation}\label{Eq:Radius_Russac}
    r_t^\gamma(\delta)\triangleq R\sqrt{2\log\frac1\delta+d\log\left(1+\frac{L^2(1-\gamma^{2t})}{\lambda d(1-\gamma^2)}\right)}.
\end{equation}
Note that when $t\to\infty$, $(1-\gamma^{2t})\to 1$, meaning that the radius is bounded.

\paragraph{Why the claimed bound cannot hold.}
Discounting gives the process a finite effective memory. The process does not become progressively more stable: old noise disappears, but new noise continues to arrive. Hence, a bounded threshold cannot control the discounted process forever. In fact, in \cref{Sec:Counterexample}, we show a counterexample to the Russac bound that supports this intuition.

\paragraph{Correction.} First of all, radius \cref{Eq:Radius_Russac} remains valid pointwise, i.e., for any $t\ge 1$,
\begin{equation}\label{Eq:fixed_t_correction}
    \forall t\ge 1,\qquad \mathbb P\left(\|S_t^\gamma\|_{(\widetilde V_t^\gamma)^{-1}}> r_t^\gamma(\delta)\right)\leq\delta.
\end{equation}
A correct time-uniform statement can be obtained by, e.g., letting
\[
    \pi_t=\frac{1}{2t(1+\log t)^2},
    \qquad\delta_t=\delta\pi_t.
\]
Indeed, $\pi_t$ is decreasing on \([1,\infty)\), so
\[
\sum_{t=1}^{\infty}\pi_t
\leq
\pi_1+\int_1^\infty\frac{\dd x}{2x(1+\log x)^2}.
\]
Here $\pi_1 = \frac{1}{2}$ and the integral equals \(1/2\): with the substitution \(u=1+\log x\), it becomes
\[
\int_1^\infty\frac{\dd x}{2x(1+\log x)^2}
=
\frac12\int_1^\infty u^{-2}\dd u
=
\frac12.
\]
Therefore, \(\sum_{t=1}^{\infty}\pi_t\leq\frac12+\frac12=1\).
Applying the pointwise result with confidence level $\delta_t$
and taking a union bound gives
\begin{equation}\label{Eq:Correct-time-uniform}
    \mathbb P\left(\exists t\geq 1:\|S_t^\gamma\|_{(\widetilde V_t^\gamma)^{-1}}> r_t^\gamma(\delta_t)\right)\leq\delta.
\end{equation}
For fixed $\gamma$ and fixed remaining parameters, the radius
\[
r_t^\gamma(\delta_t)
=
R\sqrt{2\log t+4\log(1+\log t)+O(1)}
=
R\sqrt{2\log t+O(\log\log t)}.
\]
Thus, it has order $R\sqrt{\log t}$ just as in \cref{Eq:AY-Plugin}. 

\section{A Counterexample to \cref{Eq:Russac} and Lower Bounds for Anytime Boundaries}
\label{Sec:Counterexample}

This section has two parts. We first give a scalar Gaussian example in which the proposed bounded radius in \cref{Eq:Russac} is crossed with probability one. We then establish lower bounds for anytime boundaries. The main result lower-bounds the running envelope without any monotonicity assumption; for nondecreasing boundaries, it gives a direct lower bound on the boundary value at each time. This rules out any bounded anytime radius under the same assumptions.

\subsection{A scalar Gaussian counterexample}

We now construct a one-dimensional Gaussian example in which the proposed boundary is crossed with probability one. Take \(d=1\), \(L=1\), \(\theta^*=0\), \(X_t=1\) for all $t\ge 1$, and \(\lambda>0\). Let \((\eta_t)_{t\geq1}\) be independent random variables satisfying \(\eta_t\sim\mathcal N(0,R^2)\). Then \(Y_t=\eta_t\), so this is a stationary linear model. The deterministic choice \(X_t=1\) is valid because it is \(\mathcal F_{t-1}\)-measurable and satisfies \(|X_t|\leq L=1\); in the bandit scenario, it effectively takes the action set to be only \(\{1\}\). With respect to the natural filtration, this example satisfies the assumptions of \cref{Sec:background}. Setting \(S_0^\gamma=0\), we have
\begin{equation}\label{Eq:counterexample-recursion} S_t^\gamma=\gamma S_{t-1}^\gamma+\eta_t,\qquad \widetilde V_t^\gamma=\lambda+\frac{1-\gamma^{2t}}{1-\gamma^2}. \end{equation}

The first recursion in \cref{Eq:counterexample-recursion} is a stable first-order autoregression, usually called an AR(1) process: at each time, the process retains a fraction \(\gamma\) of its previous value and receives a fresh independent Gaussian noise. Iterating the recursion gives \(S_t^\gamma=\sum_{s=1}^t\gamma^{t-s}\eta_s\), and hence
\[
\operatorname{Var}(S_t^\gamma)=R^2\frac{1-\gamma^{2t}}{1-\gamma^2}=R^2\bigl(\widetilde V_t^\gamma-\lambda\bigr)\longrightarrow\frac{R^2}{1-\gamma^2}.
\]
Thus, discounting keeps the variance uniformly bounded over time, but does not make the random fluctuations vanish. The same identity explains the term self-normalizer: apart from the factor \(R^2\) and the regularization term \(\lambda\), \(\widetilde V_t^\gamma\) is the variance scale of the random quantity \(S_t^\gamma\) that it normalizes. This calculation does not by itself prove an almost-sure boundary crossing, but it explains why a bounded time-uniform boundary appears suspicious. In fact, we will see in the next proposition that the process crosses the proposed bounded boundary with probability one. The proof is deferred to \cref{App:proof-counterexample}.

\begin{proposition}
For every \(\delta\in(0,1)\), the process in \cref{Eq:counterexample-recursion} satisfies
\[
\mathbb P\left(\exists t\geq1:\left\lVert S_t^\gamma\right\rVert_{(\widetilde V_t^\gamma)^{-1}}>r_t^\gamma(\delta)\right)=1,
\]
where \(r_t^\gamma(\delta)\) is the radius defined in \cref{Eq:Radius_Russac}.
\label{Prop:counterexample}
\end{proposition}

\subsection{Lower bounds for anytime boundaries}

The previous proposition shows that a bounded time-uniform radius cannot be valid. We first prove a finite-horizon lower bound for one common level controlling the first \(T\) times. This formulation makes no monotonicity assumption and therefore lower-bounds the running envelope of every deterministic anytime boundary. For every deterministic nondecreasing anytime boundary, the running envelope equals \(B_T(\delta)\). In particular, for fixed \(\gamma\) and \(\lambda\), there are positive constants \(c_{\gamma,\lambda}\) and \(x_0\) such that, if \(\delta\leq1/2\) and \(T/\delta\geq x_0\), every such boundary must satisfy
\[
B_T(\delta)\geq c_{\gamma,\lambda}R\sqrt{\log(T/\delta)}.
\]
The corrected radius in \cref{Eq:Correct-time-uniform} is one boundary in this class and achieves this order. We now explain the finite-horizon argument. Suppose that we have an anytime boundary \(B_t(\delta)\) such that
\[
\mathbb P\left(
\frac{\lvert S_t^\gamma\rvert}{\sqrt{\widetilde V_t^\gamma}}
\leq B_t(\delta)
\text{ for every }t\geq1
\right)\geq1-\delta.
\]

Fix a horizon \(T\). Whenever the anytime guarantee holds, it also holds at each of the first \(T\) times. Let
\[
C_T(\delta)\triangleq\max_{1\leq t\leq T}B_t(\delta).
\]
We make no monotonicity assumption on \(t\mapsto B_t(\delta)\). Thus, \(C_T(\delta)\) is simply the largest boundary value among the first \(T\) times, and it need not equal \(B_T(\delta)\). Since every \(B_t(\delta)\) is at most \(C_T(\delta)\), the normalized process also stays below the single level \(C_T(\delta)\) throughout the first \(T\) times. Therefore,
\[
\mathbb P\left(
\max_{1\leq t\leq T}
\frac{\lvert S_t^\gamma\rvert}{\sqrt{\widetilde V_t^\gamma}}
\leq C_T(\delta)
\right)\geq1-\delta.
\]

This finite-horizon statement is weaker than the original anytime guarantee, because it replaces the separate thresholds \(B_1(\delta),\ldots,B_T(\delta)\) by their largest value. However, every valid anytime boundary gives such a common level \(C_T(\delta)\). The next proposition gives a lower bound on \(C_T(\delta)\).
\begin{proposition}[Finite-horizon lower bound]
\label{Prop:finite-horizon-lower-bound}
Consider the same scalar Gaussian process as in \cref{Eq:counterexample-recursion}. Fix \(T\geq1\) and \(\delta\in(0,1)\). Let \(C_T(\delta)\) be any deterministic boundary satisfying
\[
\mathbb P\left(
\max_{1\leq t\leq T}
\frac{\lvert S_t^\gamma\rvert}{\sqrt{\widetilde V_t^\gamma}}
\leq C_T(\delta)
\right)\geq1-\delta .
\]
Then
\[
C_T(\delta)
\geq
\frac{R}{(1+\gamma)\sqrt{\lambda+(1-\gamma^2)^{-1}}}
\,
\Phi^{-1}\left(\frac{1+(1-\delta)^{1/T}}{2}\right),
\]
where \(\Phi\) is the standard normal distribution function. In particular, for fixed \(\gamma\) and \(\lambda\), there are positive constants \(c_{\gamma,\lambda}\) and \(x_0\) such that, for \(\delta\in(0,1/2]\) and \(T/\delta\geq x_0\),
\[
C_T(\delta)\geq c_{\gamma,\lambda}R\sqrt{\log\frac{T}{\delta}}.
\]
\end{proposition}

Thus, for an arbitrary deterministic anytime boundary, the proposition lower-bounds its running envelope \(C_T(\delta)=\max_{1\leq t\leq T}B_t(\delta)\), rather than any particular value \(B_T(\delta)\).

The proof is deferred to \cref{App:proof-finite-horizon-lower-bound}. Its main idea is simple. Since \(\widetilde V_t^\gamma\leq V_{\max}\), the normalized bound in the proposition implies that \(\lvert S_t^\gamma\rvert\) must remain below the common level \(C_T(\delta)\sqrt{V_{\max}}\) for all \(1\leq t\leq T\). We can therefore ask how large this deterministic interval must be in order to contain the process throughout the first \(T\) times with probability at least \(1-\delta\). The recursion then implies that each of the \(T\) independent noises must lie in a corresponding interval. By independence, the probability that one Gaussian noise lies outside this interval must be at most of order \(\delta/T\). This is possible only when \(C_T(\delta)\) has order at least \(R\sqrt{\log(T/\delta)}\).

The finite-horizon formulation is a more general result: it gives the running-envelope lower bound without imposing monotonicity. The following corollary gives the more familiar pointwise consequence for nondecreasing deterministic anytime boundaries. The corrected boundary \(r_t^\gamma(\delta_t)\) in \cref{Eq:Correct-time-uniform} is nondecreasing: \(\pi_t\) decreases with \(t\), so \(\log(1/\delta_t)\) increases, and \(1-\gamma^{2t}\) also increases. Hence, both terms inside the defining radius are nondecreasing, and the corollary applies to this correction.

\begin{corollary}[Nondecreasing anytime boundaries]
Under the same assumptions, suppose \(B_t(\delta)\) is a nondecreasing deterministic anytime boundary satisfying
\[
\mathbb P\left(
\frac{\lvert S_t^\gamma\rvert}{\sqrt{\widetilde V_t^\gamma}}
\leq B_t(\delta)
\text{ for every }t\geq1
\right)\geq1-\delta,
\]
Then there are positive constants \(c_{\gamma,\lambda}\) and \(x_0\) such that, if \(\delta\in(0,1/2]\) and \(t/\delta\geq x_0\), then
\[
B_t(\delta)\geq c_{\gamma,\lambda}R\sqrt{\log\frac{t}{\delta}}.
\]
\end{corollary}
\begin{proof}
Set \(C_t(\delta)=\max_{1\leq s\leq t}B_s(\delta)\). By nondecreasingness, \(C_t(\delta)=B_t(\delta)\). The anytime guarantee implies the finite-horizon common-boundary condition for \(C_t(\delta)\), so \cref{Prop:finite-horizon-lower-bound} completes the proof.
\end{proof}

Together, these two propositions rule out the proposed bounded radius. The first proposition shows that the original radius is eventually crossed with probability one. The second gives a lower bound for every horizon \(T\) and, when \(\delta\leq1/2\) and \(T/\delta\) is sufficiently large, shows that the running envelope of any valid anytime boundary must be at least of order \(R\sqrt{\log(T/\delta)}\). Thus, the union-bound correction \cref{Eq:Correct-time-uniform} matches the lower-bound order in this scalar Gaussian example. Since this correction is nondecreasing, the corollary also shows that its \(\sqrt{\log T}\) growth is unavoidable among nondecreasing deterministic anytime boundaries. In particular, this scalar Gaussian process already rules out the claimed bounded anytime radius under the assumptions of \cref{Sec:background}. Therefore, no correct stopping-time argument can establish that bounded radius for the full class of processes covered by those assumptions.

\section{The Classical Self-Normalized Argument}\label{Sec:SN-technique-revisit}
In this section, we revisit the proof of the self-normalized concentration inequality of \citet{APS11-LinUCB} using Ville's inequality. The proof consists of three steps: constructing an exponential supermartingale for each fixed direction, mixing these supermartingales over a fixed Gaussian distribution, and applying Ville's inequality to the resulting mixture process.


\subsection{Ville's inequality}

We first recall Ville's inequality used in the proof.

\begin{lemma}[Ville's inequality]
\label{Lem:Ville}
Let $(M_t)_{t\geq0}$ be a nonnegative supermartingale with respect to $(\mathcal F_t)_{t\geq0}$. Then, for every $u>0$,
\[
    \P\lr{\exists t\geq0:M_t\geq u}\leq\frac{\E[M_0]}{u}.
\]
\end{lemma}
For any fixed deterministic time \(t\), Markov's inequality gives
\[
\P\lr{M_t\geq u}\leq\frac{\E[M_t]}{u}.
\]
If \((M_t)_{t\geq0}\) is a nonnegative supermartingale, Ville's inequality controls the probability that the process ever crosses the threshold \(u\). Its application requires one nonnegative supermartingale indexed by time, rather than a collection of random variables or mixture constructions that are valid only separately at each time.


\subsection{Exponential supermartingales in a fixed direction}

Recall that
\[
    S_t=\sum_{s=1}^t\eta_sX_s,\qquad
    V_t=\lambda I_d+\sum_{s=1}^tX_sX_s^\top,
\]
where $\lambda>0$ is fixed. For any fixed
$x\in\R^d$, define
\begin{equation}
\label{Eq:Directional-Supermartingale}
    M_t(x)=\exp\left(\frac{x^\top S_t}{R}-\frac{1}{2}x^\top(V_t-\lambda I_d)x\right).
\end{equation}
Since $X_t$ is $\mathcal F_{t-1}$-measurable and $\eta_t$ is
conditionally $R$-sub-Gaussian,
\begin{equation}\label{Eq:Property:Directional-Supermartingale}
    \E\left[M_t(x)\mid\mathcal F_{t-1}\right]
=M_{t-1}(x)\exp\left(-\frac{1}{2}(x^\top X_t)^2\right)\E\left[\exp\left(\frac{\eta_t x^\top X_t}{R}\right)\middle|\mathcal F_{t-1}\right]
\leq M_{t-1}(x).
\end{equation}
Therefore, $(M_t(x))_{t\geq0}$ is a nonnegative supermartingale for every fixed direction $x$.

\subsection{Gaussian mixture of supermartingales}\label{Subsec:recipe:mixture}

To control all directions simultaneously, let $h$ be the density
of the fixed Gaussian distribution $\mathcal N(0,\lambda^{-1}I_d)$ and
define
\[
    M_t=\int_{\R^d}M_t(x)h(x)\, dx.
\]
Because $M_t(x)h(x)\geq0$, conditional Tonelli's theorem allows us to exchange conditional expectation and integration. Hence, the mixture \((M_t)_{t\geq0}\) is also a nonnegative supermartingale:
\[
\mathbb E[M_t\mid\mathcal F_{t-1}]=\int_{\R^d}\mathbb E[M_t(x)\mid\mathcal F_{t-1}]h(x)\,dx\leq\int_{\R^d}M_{t-1}(x)h(x)\,dx=M_{t-1},
\]
with $M_0=1$. Substituting the Gaussian density and $M_t(x)$ in \cref{Eq:Directional-Supermartingale} gives
\begin{equation}\label{Eq:Mixture-Closed-Form}
M_t=\frac{\lambda^{d/2}}{(2\pi)^{d/2}}\int_{\R^d}\exp\left(\frac{x^\top S_t}{R}-\frac{1}{2}x^\top V_tx\right)\ dx
=\left(\frac{\lambda^d}{\det(V_t)}\right)^{1/2}\exp\left(\frac{\lrn{S_t}_{V_t^{-1}}^2}{2R^2}\right).
\end{equation}
Therefore,
\[
\P\lrc{\exists t\ge 0:\lrn{S_t}_{V_t^{-1}}>R\sqrt{2\log\frac{\det(V_t)^{1/2}}{\delta\lambda^{d/2}}}}
=\P\lrc{\exists t\ge 0:M_t>\frac{1}{\delta}} \le \delta,
\]
where the inequality is due to applying Ville's inequality to the single mixture nonnegative supermartingale $(M_t)_{t\geq0}$ with $M_0=1$.


\subsection{A Valid Extension for Fixed Predictable Weights}\label{Subsec:valid_weighted_extension}

The preceding argument remains valid for one predictable sequence of weights and a fixed regularizer. This valid special case is also
contained in the weighted setup considered by
\citet{RVC19weightedLB}. In particular, let $(w_t)_{t\geq1}$ be a real-valued predictable process, so that $w_t$ is $\mathcal F_{t-1}$-measurable, and define
\[
    S_t^w=\sum_{s=1}^t w_s\eta_sX_s,
    \qquad
    V_t^w=\lambda I_d+\sum_{s=1}^t w_s^2X_sX_s^\top,
\]
where $\lambda>0$ is fixed. Since $w_tX_t$ remains $\mathcal F_{t-1}$-measurable, the previous result can be applied with $X_t$ replaced by $w_tX_t$. Therefore,
\begin{equation}
\label{Eq:Fixed-Predictable-Weights}
\P\lr{
    \exists t\geq0:
    \lrn{S_t^w}_{(V_t^w)^{-1}}>R\sqrt{2\log\frac{\det(V_t^w)^{1/2}}{\delta\lambda^{d/2}}}
}\leq\delta.
\end{equation}

Thus, predictable weighting by itself does not destroy the time-uniform argument. In \cref{Eq:Fixed-Predictable-Weights}, $(w_s)_{s\geq1}$ is one sequence indexed by observation time: once observation $s$ is assigned weight $w_s$, that weight does not change as the terminal time continues. Also, both the regularizer $\lambda I_d$ and the Gaussian mixing distribution $\mathcal N(0,\lambda^{-1}I_d)$ remain fixed over time. 

\section{Where the Time-Uniform Argument Fails}\label{Sec:Russac}

Recall the discounted quantities
\[
S_t^\gamma=\sum_{s=1}^t\gamma^{t-s}\eta_sX_s,\qquad \widetilde V_t^\gamma=\lambda I_d+\sum_{s=1}^t\gamma^{2(t-s)}X_sX_s^\top.
\]
The weights \(\gamma^{t-s}\) depend on the terminal time \(t\), so they do not form one predictable sequence to which \cref{Eq:Fixed-Predictable-Weights} can be applied simultaneously over time. The argument of \citet{RVC19weightedLB} instead uses an equivalent rescaling that fixes the observation-time weights but makes the regularizer depend on \(t\). Define
\begin{equation*}
\overline S_t^\gamma\triangleq\sum_{s=1}^t\gamma^{-s}\eta_sX_s,\qquad \overline V_t^\gamma\triangleq\gamma^{-2t}\lambda I_d+\sum_{s=1}^t\gamma^{-2s}X_sX_s^\top. \end{equation*}
Then \(S_t^\gamma=\gamma^t\overline S_t^\gamma\) and \(\widetilde V_t^\gamma=\gamma^{2t}\overline V_t^\gamma\). The observation-time weights \(\gamma^{-s}\) are now fixed, but the regularizer \(\gamma^{-2t}\lambda I_d\) varies with the terminal time.

For every fixed \(x\in\R^d\), define
\begin{equation}\label{Eq:Russac-Directional-Supermartingale} M_t(x)=\exp\left(\frac{x^\top\overline S_t^\gamma}{R}-\frac{1}{2}x^\top\left(\overline V_t^\gamma-\gamma^{-2t}\lambda I_d\right)x\right). \end{equation}
Since \(\overline V_t^\gamma-\gamma^{-2t}\lambda I_d=\sum_{s=1}^t\gamma^{-2s}X_sX_s^\top\), this is exactly the fixed-weight construction in \cref{Eq:Directional-Supermartingale} with \(X_s\) replaced by \(\gamma^{-s}X_s\). Therefore, \((M_t(x))_{t\geq0}\) is a nonnegative supermartingale with \(M_0(x)=1\) for every fixed direction \(x\).

\subsection{The mixtures do not form one single supermartingale}

For each \(t\geq0\), let \(h_t\) be the density of \(\mathcal N(0,\gamma^{2t}\lambda^{-1}I_d)\) and define
\[
\overline M_t\triangleq\int_{\R^d}M_t(x)h_t(x)\,dx.
\]
Substituting the definitions of $M_t(x)$ and $h_t(x)$, the integral
\begin{equation}\label{Eq:Russac-Mixture-Closed-Form}
\overline M_t=\left(\frac{\gamma^{-2td}\lambda^d}{\det(\overline V_t^\gamma)}\right)^{1/2}\exp\left(\frac{\|\overline S_t^\gamma\|_{(\overline V_t^\gamma)^{-1}}^2}{2R^2}\right)=\left(\frac{\lambda^d}{\det(\widetilde V_t^\gamma)}\right)^{1/2}\exp\left(\frac{\|S_t^\gamma\|_{(\widetilde V_t^\gamma)^{-1}}^2}{2R^2}\right).
\end{equation}
If \((\overline M_t)_{t\geq0}\) were one nonnegative supermartingale with \(\overline M_0=1\), Ville's inequality would provide the desired time-uniform result.

The problem is that the Gaussian density used in the mixture changes with \(t\). Conditional Tonelli's theorem and the directional supermartingale property \cref{Eq:Property:Directional-Supermartingale} only give
\[
\mathbb E[\overline M_t\mid\mathcal F_{t-1}]=\int_{\R^d}\mathbb E[M_t(x)\mid\mathcal F_{t-1}]h_t(x)\,dx\leq\int_{\R^d}M_{t-1}(x)h_t(x)\,dx,
\]
whereas
\[
\overline M_{t-1}=\int_{\R^d}M_{t-1}(x)h_{t-1}(x)\,dx.
\]
The two expressions integrate the same random function \(M_{t-1}(x)\) against different densities. Since \(h_t\neq h_{t-1}\), the first integral is not necessarily bounded by \(\overline M_{t-1}\). The directional supermartingale property compares \(M_t(x)\) and \(M_{t-1}(x)\) for the same fixed direction \(x\); it does not compare mixtures formed using different densities. Therefore, the construction does not establish
\[
\mathbb E[\overline M_t\mid\mathcal F_{t-1}]\leq\overline M_{t-1},
\]
and the fixed-time mixtures \((\overline M_t)_{t\geq0}\) do not, in general, form one nonnegative supermartingale to which Ville's inequality can be applied. Of course, this does not say the inequality is false and maybe there is a way to bypass it. However, the counterexample in \cref{Sec:Counterexample} shows that this is not just a missing proof: the conclusion indeed is false.

\subsection{The stopping-time repair is invalid}

Lemma~3 of \citet{RVC19weightedLB} attempts to bypass the missing supermartingale property using an auxiliary sequence \((Z_t)_{t\geq0}\), independent of the data-generating process, such that \(Z_t\) has density \(h_t\). Let \(\tau\) be a stopping time with respect to \((\mathcal F_t)_{t\geq0}\); for simplicity, one may first take \(\tau\) to be bounded. Independence of the auxiliary Gaussian sequence gives
\[
\mathbb E[\overline M_\tau]=\mathbb E[M_\tau(Z_\tau)].
\]
For every fixed \(x\in\R^d\), optional stopping applied to the nonnegative supermartingale \((M_t(x))_{t\geq0}\) gives
\[
\mathbb E[M_\tau(x)]\leq1.
\]
The proof then conditions on the entire Gaussian sequence \((Z_t)_{t\geq0}\) and attempts to apply this fixed-direction inequality to \(M_\tau(Z_\tau)\). The difficulty is that conditioning fixes the whole sequence of directions, but it does not produce one common direction. Indeed, conditionally on \((Z_t)_{t\geq0}=(z_t)_{t\geq0}\),
\[
M_\tau(Z_\tau)=\sum_{t\geq0}\mathbf 1\{\tau=t\}M_t(z_t),
\]
whereas, for one fixed \(x\),
\[
M_\tau(x)=\sum_{t\geq0}\mathbf 1\{\tau=t\}M_t(x).
\]
In the first expression, the direction is \(z_t\) on the branch \(\{\tau=t\}\); in the second, the same direction \(x\) is used on every possible branch. Conditioning makes each \(z_t\) deterministic, but it does not make the sequence \((z_t)_{t\geq0}\) constant. Equivalently, the time-indexed process \(M_t(z_t)\) is not covered by the fixed-direction supermartingale result. Hence,
\[
\forall x\in\R^d,\qquad \mathbb E[M_\tau(x)]\leq1
\]
does not imply
\[
\mathbb E[M_\tau(Z_\tau)]\leq1.
\]
The proof then chooses \(\tau\) as the first time at which the proposed self-normalized boundary is crossed (over a finite horizon \(T\), its bounded truncation \(\tau\wedge T\))and applies Markov's inequality using the claimed bound \(\mathbb E[\overline M_\tau]\leq1\). Since the expectation bound has not been established, the stopping-time argument does not prove the claimed time-uniform inequality.

\subsection{What remains valid at a fixed deterministic time}

Although the mixtures do not form one supermartingale over time, each mixture remains valid at a fixed deterministic time. Fix \(t\geq1\). Since \(h_t\) is then one fixed probability density, the same Tonelli argument as in \cref{Subsec:recipe:mixture}, together with \(\mathbb E[M_t(x)]\leq1\) for every fixed \(x\), gives
\[
\mathbb E[\overline M_t]=\int_{\R^d}\mathbb E[M_t(x)]h_t(x)\,dx\leq\int_{\R^d}h_t(x)\,dx=1.
\]
Combining this with \cref{Eq:Russac-Mixture-Closed-Form} and applying Markov's inequality gives, for every fixed deterministic \(t\),
\[
\mathbb P\left(\|S_t^\gamma\|_{(\widetilde V_t^\gamma)^{-1}}>R\sqrt{2\log\frac{\det(\widetilde V_t^\gamma)^{1/2}}{\delta\lambda^{d/2}}}\right)\leq\delta.
\]
The determinant bound in \cref{Sec:background} yields
\[
\mathbb P\left(\|S_t^\gamma\|_{(\widetilde V_t^\gamma)^{-1}}>r_t^\gamma(\delta)\right)\leq\delta.
\]
Thus, the valid conclusion is that the inequality holds separately for every fixed deterministic \(t\).

\section{Related Work and Downstream Consequences}\label{Sec:Impact}

The literature presents a mixed picture. Some analyses apply the claimed anytime inequality directly, whereas others retain a fixed-time statement or explicitly pay for simultaneous control over a finite horizon. The latter treatments suggest that the difficulty created by the time index was likely recognized in parts of the literature. However, in preparing this note, we did not find a prior written discussion that both isolates the invalid stopping-time step and establishes that the resulting anytime claim is false. We therefore record this distinction here, especially because the claim has continued to appear in later analyses; the examples below are not intended to be exhaustive.

\subsection{Scope of the counterexample and available repairs}\label{Subsec:scope:direct-implication}

The counterexample in \cref{Sec:Counterexample} disproves the time-uniform weighted self-normalized inequality in \cref{Eq:Russac}, whose bounded radius contains no cost for time-uniformity. Therefore, every downstream proof that uses \cref{Eq:Russac} as one event holding simultaneously for all times has a gap at that step. This does not by itself imply that the underlying algorithm or the polynomial order of its final regret bound is incorrect.

The pointwise inequality in \cref{Eq:fixed_t_correction} remains valid. For a known finite horizon \(T\), applying it at confidence level \(\delta/T\) and taking a union bound gives
\[
\mathbb P\left(\exists 1\leq t\leq T:\|S_t^\gamma\|_{(\widetilde V_t^\gamma)^{-1}}>r_t^\gamma(\delta/T)\right)\leq\delta.
\]
Relative to \(r_t^\gamma(\delta)\), this repair adds \(2\log T\) inside the expression under the square root. For an infinite horizon, the allocation in \cref{Eq:Correct-time-uniform} gives a radius of order \(\sqrt{\log t}\). 
If such replacement of the confidence radius is the only required modification, the correction changes only logarithmic factors and hence does not change the polynomial order hidden by \(\widetilde O\)-notation. Whether this is sufficient should be checked in each application, since the confidence radius may also affect the algorithm or its tuning.

\subsection{Direct downstream uses}

For each direct use below, we state the problem studied and the step that uses the invalid simultaneous inequality.

\paragraph{\citet{RCG20-NSGLB}.} This paper studies non-stationary generalized linear bandits with piecewise changes and develops sliding-window and discounted UCB algorithms. Its Corollary 5 restates the discounted time-uniform inequality of \citet{RVC19weightedLB}. The proof of Proposition 2 applies Corollary 5 to control the weighted noise term simultaneously for all \(t\); Proposition 2 then yields the confidence bound used in Corollary 3 and the regret analysis of the discounted algorithm. The paper explicitly explains that, unlike the sliding-window analysis, the discounted analysis contains \(\log(1/\delta)\) rather than \(\log(T/\delta)\) because Corollary 5 is an anytime deviation inequality. Thus, both the simultaneous noise-control event and the resulting removal of the time-union cost are not justified.

\paragraph{\citet{KT20}.} This paper develops randomized-exploration algorithms for non-stationary stochastic linear bandits. Its Lemma 4 restates Proposition 3 of \citet{RVC19weightedLB} as a confidence event \(E_{\mathrm{wls}}\) claimed to hold simultaneously over \(t\in[T]\). This event is then used in the regret analyses of D-RandLinUCB and D-LinTS.

\paragraph{\citet{TV20}.} This paper studies episodic reinforcement learning in non-stationary linear MDPs, where both the rewards and transition kernels are linear in known features and may evolve over episodes. Its Lemma 10 states the weighted self-normalized inequality with a deterministic time-varying regularizer and a simultaneous-in-\(t\) conclusion. The proof of Lemma 13 applies Lemma 10 to obtain uniform concentration for the estimated transition model; the resulting confidence control is then used in the proof of Theorem 1 and its dynamic-regret consequence in Corollary 1.

\paragraph{\citet{FRAC21-GLBDrift}.} This paper studies non-stationary generalized linear bandits. The proof of Lemma 1 directly applies Proposition 1 of \citet{RVC19weightedLB} to obtain a discounted noise bound holding for all \(t\geq1\). Lemma 1 is then used in Lemma 5 to control the prediction error, and this control enters the regret analysis in Theorem 1.

\paragraph{\citet{DZKT+22-WPB}.} This paper studies non-stationary Gaussian-process bandits and uses weighted Gaussian-process regression to discount old observations. Its Lemmas 10 and 11 claim a weighted self-normalized bound holding simultaneously for all \(t\), obtained by adapting the standard fixed-regularizer argument to the time-dependent regularizer \(\alpha_t I\) and then expressing the result in terms of weighted information gain. Lemma 9 shows that this framework recovers the weighted linear bandit of \citet{RVC19weightedLB} as a linear-feature special case. The resulting concentration bound is used in Theorems 2 and 3; Theorem 3 then enters the dynamic-regret analysis of Theorem 4.

\paragraph{\citet{WZZ23revisiting}.} This paper revisits weighted strategies for non-stationary linear, generalized-linear, and self-concordant bandits. Its Theorem 5 restates the claimed weighted self-normalized inequality for a positive weight sequence \((w_s)_{s\geq1}\) and time-varying scalar regularizers \((\mu_t)_{t\geq1}\). In the proof of Lemma 5, the authors set \(w_s=\gamma^{t-s-1}\) and \(\mu_t=\lambda\) and use Theorem 5 to obtain a bound holding for all \(t\). Since this choice of \(w_s\) depends on the terminal time, it does not correspond to one fixed weight sequence across time, so the simultaneous conclusion does not directly follow from Theorem 5. Lemma 5 is then used in Lemma 1 and Theorem 1 for linear bandits; the analogous generalized-linear-bandit chain is Lemma 7, Lemma 2, and Theorem 2. In the self-concordant-bandit analysis, Theorem 6 is stated for a fixed time, while Lemma 3 uses it to obtain a conclusion for all \(t\), which then enters Theorem 3.

\paragraph{\citet{WLZ26-CoRE}.} This paper studies how limited computational resources should be allocated among multiple learning tasks whose losses must reach prescribed targets by their deadlines. Because each task's loss curve changes as training progresses, the authors use discounted least squares to estimate its current parameters, giving less weight to older observations that are less representative of its present progress. Theorem 2 in the appendix restates Theorem 1 of \citet{RVC19weightedLB}. Its proof sets \(w_i=\gamma^{n-i}\) and \(\mu_n=\lambda\) and then asserts that the resulting confidence bound holds for all \(n\). Since this choice of \(w_i\) depends on the terminal time \(n\), the simultaneous conclusion does not directly follow from the stated result. The resulting radius \(\beta_n\) is then used to guide the resource-allocation decision.

\subsection{Fixed-time and finite-horizon control}

The following works use fixed-time concentration or explicit finite-horizon confidence allocation in their final analyses. They therefore avoid relying on the claimed anytime radius, but do not provide a new anytime inequality.

\paragraph{\citet{RFCG21-GLB}.} This paper studies self-concordant generalized linear bandits, including logistic and Poisson bandits, with forgetting implemented through a sliding window or exponential weights. Remark 2 notes that time-dependent regularization destroys the conditional supermartingale relation available with a fixed regularizer. Remark 3 further explains that the terminal-time dependence prevents the standard stopping-time argument from being applied. The paper therefore states its deviation result only at a fixed deterministic time and requires a union bound to control the entire trajectory.

\paragraph{\citet{WWAK25-WSB}.} This paper studies non-stationary linear contextual bandits through weighted sequential Bayesian inference, maintaining a posterior distribution over the time-varying reward parameter. Appendix B of an earlier preprint version (arXiv v3) contains the same stopping-time error. In the later UAI 2026 version, the proof of Lemma 5 instead applies a fixed-time inequality at confidence level \(\delta/T\) over the finite horizon and takes a union bound. The latest version therefore obtains simultaneous control by explicitly paying for the finite horizon rather than treating the fixed-time radius as an anytime radius.

\section*{Acknowledgments}
This work was supported by the Academia Sinica Postdoctoral Scholar Program, Grant No.~AS-PD-1151-M15-2. The author thanks Melih Kandemir, Pei-Yuan Wu, Po-An Wang, Julian Zimmert, Yanlin Chen, and Nicklas Werge for helpful feedback and discussions on earlier versions of this note. The author also thanks the Institute of Statistics and Data Science, National Tsing Hua University, for providing access to workspace and facilities.


\small

\bibliographystyle{abbrvnat}
\bibliography{Yi-Shan-2025.bib}

\normalsize



\appendix

\section{Proofs for \cref{Sec:Counterexample}}

\subsection{Proof of \cref{Prop:counterexample}}
\label{App:proof-counterexample}

\begin{proof}
In one dimension, \(\left\lVert S_t^\gamma\right\rVert_{(\widetilde V_t^\gamma)^{-1}}=\lvert S_t^\gamma\rvert/\sqrt{\widetilde V_t^\gamma}\). Thus, the proposed boundary is crossed at time \(t\) precisely when
\[
\lvert S_t^\gamma\rvert>r_t^\gamma(\delta)\sqrt{\widetilde V_t^\gamma}.
\]
We first bound these time-dependent boundaries $r_t^\gamma(\delta)\sqrt{\widetilde V_t^\gamma}$ by one finite constant \(C\), and then show that the process $(S_t)_{t\ge 1}$ cannot remain within \([-C,C]\) forever.

Since \(r_t^\gamma(\delta)\) is bounded as discussed in \cref{Sec:background}, and \(\widetilde V_t^\gamma\to\lambda+(1-\gamma^2)^{-1}\), the constant
\[
C\triangleq\sup_{t\geq1}r_t^\gamma(\delta)\sqrt{\widetilde V_t^\gamma}
\]
is finite. By definition, \(r_t^\gamma(\delta)\sqrt{\widetilde V_t^\gamma}\leq C\) for every \(t\geq1\). Consequently,
\[
\lvert S_t^\gamma\rvert>C\quad\Longrightarrow\quad\left\lVert S_t^\gamma\right\rVert_{(\widetilde V_t^\gamma)^{-1}}>r_t^\gamma(\delta).
\]
It therefore suffices to prove
\[
\mathbb P\left(\exists t\geq1:\lvert S_t^\gamma\rvert>C\right)=1.
\]

Let \(q_C\triangleq\mathbb P\bigl(\lvert\eta_1\rvert>(1+\gamma)C\bigr)>0\), where positivity is due to \(C<\infty\) and a nondegenerate Gaussian random variable has unbounded support. For any \(m\geq1\), suppose that \(\max_{1\leq t\leq m}\lvert S_t^\gamma\rvert\leq C\). Since \(S_0^\gamma=0\), the recursion implies, for every \(1\leq t\leq m\),
\[
\lvert\eta_t\rvert=\lvert S_t^\gamma-\gamma S_{t-1}^\gamma\rvert\leq\lvert S_t^\gamma\rvert+\gamma\lvert S_{t-1}^\gamma\rvert\leq(1+\gamma)C.
\]
Therefore,
\[
\left\{\max_{1\leq t\leq m}\lvert S_t^\gamma\rvert\leq C\right\}\subseteq\bigcap_{t=1}^m\left\{\lvert\eta_t\rvert\leq(1+\gamma)C\right\}.
\]
Using independence of the noises,
\[
\mathbb P\left(\max_{1\leq t\leq m}\lvert S_t^\gamma\rvert\leq C\right)\leq\mathbb P\left(\bigcap_{t=1}^m\left\{\lvert\eta_t\rvert\leq(1+\gamma)C\right\}\right)=(1-q_C)^m\longrightarrow0.
\]

The finite-horizon events on the left decrease as \(m\) increases, and their intersection is the event that the process remains in \([-C,C]\) forever. Hence,
\[
\mathbb P\left(\lvert S_t^\gamma\rvert\leq C\text{ for every }t\geq1\right)=\lim_{m\to\infty}\mathbb P\left(\max_{1\leq t\leq m}\lvert S_t^\gamma\rvert\leq C\right)=0.
\]
Taking complements gives \(\mathbb P\left(\exists t\geq1:\lvert S_t^\gamma\rvert>C\right)=1\).
\end{proof}

\subsection{Proof of \cref{Prop:finite-horizon-lower-bound}}
\label{App:proof-finite-horizon-lower-bound}

\begin{proof}
Let \(V_{\max}\triangleq \lambda+\frac{1}{1-\gamma^2}\). Whenever the maximum in the proposition's assumption is at most \(C_T(\delta)\), we have, for every \(1\leq t\leq T\),
\[
\lvert S_t^\gamma\rvert
\leq C_T(\delta)\sqrt{\widetilde V_t^\gamma}
\leq C_T(\delta)\sqrt{V_{\max}}.
\]
The last inequality uses \(\widetilde V_t^\gamma\leq V_{\max}\). Hence, the assumption of $C_T(\delta)$ in the proposition implies
\[
\mathbb P\left(
\max_{1\leq t\leq T}\lvert S_t^\gamma\rvert
\leq C_T(\delta)\sqrt{V_{\max}}\right)\geq1-\delta.
\]

The bound on the right-hand side is deterministic, while the recursion depends on \(T\) independent Gaussian noises. As in the proof of \cref{Prop:counterexample},
\[
\max_{1\leq t\leq T}\lvert S_t^\gamma\rvert\leq C_T(\delta)\sqrt{V_{\max}} \quad \implies \quad \lvert \eta_t\rvert\leq(1+\gamma)C_T(\delta)\sqrt{V_{\max}} \quad \text{for every } 1\leq t\leq T.
\]
Thus, by independence of the noises,
\[
1-\delta
\leq
\mathbb P\left(\max_{1\leq t\leq T}\lvert S_t^\gamma\rvert
\leq C_T(\delta)\sqrt{V_{\max}}\right)
\leq
\left[\mathbb P\left(\lvert\eta_1\rvert\leq(1+\gamma)C_T(\delta)\sqrt{V_{\max}}\right)\right]^T.
\]
Thus, to make this probability at least \(1-\delta\), the probability for each noise to remain in that interval must be at least \((1-\delta)^{1/T}\). This condition will force \(C_T(\delta)\) to be large. Since \(\eta_1/R\sim\mathcal N(0,1)\), this implies
\[
2\Phi\left(\frac{(1+\gamma)C_T(\delta)\sqrt{V_{\max}}}{R}\right)-1
\geq
(1-\delta)^{1/T},
\]
where we recall that for a standard normal random variable \(Z\), $\Phi(z)$ is the cumulative distribution function, and $2\Phi(z)-1$ is the probability that \(\lvert Z\rvert\leq z\). Solving for \(C_T(\delta)\) gives the displayed lower bound.

To obtain the order statement, we first let $a_T\triangleq\frac{(1+\gamma)C_T(\delta)\sqrt{V_{\max}}}{R}$. The previous formula can be rewritten as
\[
\Phi(a_T)
\geq
\frac{1+(1-\delta)^{1/T}}{2}.
\]
Since \(\mathbb P(Z>a_T)=1-\Phi(a_T)\), taking complements gives
\[
\mathbb P(Z>a_T)
\leq
\frac{1-(1-\delta)^{1/T}}{2}
\leq \frac{-\log(1-\delta)}{2T}
\leq \frac{\delta}{T},
\]
where we used \(1-e^{-x}\leq x\) and \(-\log(1-\delta)\leq2\delta\) for \(\delta\in(0,1/2]\). Thus, the threshold \(a_T\) must make the one-sided Gaussian tail no larger than \(\delta/T\). A standard Gaussian tail lower bound shows that there are universal constants \(c>0\) and \(u_0>0\) such that, whenever \(u\leq u_0\),
\[
\mathbb P(Z>a)\leq u
\qquad\Longrightarrow\qquad
a\geq c\sqrt{\log\frac{1}{u}}.
\]
When \(T/\delta\) is sufficiently large, we may apply this bound with \(u=\delta/T\). It gives
\[
a_T\geq c\sqrt{\log\frac{T}{\delta}}.
\]
Recalling the definition of \(a_T\) proves the final claim.
\end{proof}

\end{document}